\documentclass[letterpaper, 10 pt, conference]{ieeeconf}

\IEEEoverridecommandlockouts                              
\usepackage{amsmath,amssymb,amsfonts}
\usepackage{algorithmic}
\usepackage{algorithm}
\usepackage{graphicx}
\usepackage{textcomp}
\usepackage{booktabs}
\usepackage{comment}
\usepackage{xcolor}
\usepackage{dsfont}
\usepackage{bm}
\usepackage{balance}
\newcommand{\ind}{\mathds{1}}

\makeatletter
\@ifundefined{proof}{}{ }
\makeatother

\usepackage[fancythm,fancybb]{jphmacros2e}

\title{\LARGE \bf
Robust Multi-Agent Reinforcement Learning for Small UAS Separation Assurance under
GPS Degradation and Spoofing
}

\author{Alex Zongo$^{\dagger}$, Filippos Fotiadis$^{\star}$, Ufuk Topcu$^{\star}$ and Peng Wei$^{\dagger}$
\thanks{This work was supported by NASA under grant 80NSSC24M0070, and by ONR under grant N00014-22-1-2703.}
\thanks{$^{\dagger}$A. Zongo and P. Wei are with the Department of Mechanical and Aerospace Engineering, George Washington University, Washington, D.C. 20052, USA. Email: {\tt\small\{a.zongo, pwei\}@gwu.edu}.}
\thanks{$^{\star}$F. Fotiadis and U. Topcu are with the Oden Institute for Computational Engineering \& Sciences, University of Texas at Austin, Austin, TX 78712, USA. Email:  \tt\small\{ffotiadis, utopcu\}@utexas.edu.}
}

\begin{document}

\maketitle
\thispagestyle{empty}
\pagestyle{empty}

\begin{abstract}

We address robust separation assurance for small Unmanned Aircraft Systems (sUAS) under GPS degradation and spoofing via Multi-Agent Reinforcement Learning (MARL). In cooperative surveillance, each aircraft (or agent) broadcasts its GPS-derived position; when such position broadcasts are corrupted, the entire observed air traffic state becomes unreliable. We cast this state observation corruption as a zero-sum game between the agents and an adversary: with probability R, the adversary perturbs the observed state to maximally degrade each agent's safety performance. We derive a closed-form expression for this adversarial perturbation, bypassing the iterative inner optimization of adversarial training entirely and enabling linear-time evaluation in the state dimension. We show that this expression approximates the exact minimizer of the value function over the modeled uncertainty set with second-order accuracy. We further bound the safety performance gap between clean and corrupted observations, showing that it degrades at most linearly with the corruption probability under Kullback-Leibler regularization.
Finally, we integrate the closed-form adversarial policy into a MARL policy gradient algorithm to obtain a robust counter-policy for the agents. In a high-density sUAS simulation, we observe near-zero collision rates under corruption levels up to 35\%, outperforming a baseline policy trained without adversarial perturbations.

\end{abstract}

\section{Introduction}

The proliferation of small Unmanned Aircraft Systems (sUAS) in urban airspaces is reshaping logistics and emergency response. Commercial operators, including Zipline and Wing, now conduct beyond-visual-line-of-sight package deliveries under FAA authorization \cite{faa_part135_2023,marquand_bvlos_dallas_2024}. As these operations scale, autonomous separation assurance becomes imperative: vehicles must avoid conflicts without human intervention, relying entirely on onboard sensing and cooperative information exchange.

This autonomy, however, rests on fragile infrastructure. Small UAS depend on Global Positioning System (GPS) for position and velocity estimation, yet urban canyons induce multipath errors of 10--60\,m \cite{Peretic2025gnsserrors, Gutierrez2024multipath}, and GPS signals are vulnerable to spoofing attacks with field-demonstrated success rates of 5--40\% \cite{HarshadUAVTakeOver, AndrewUAVcontrolViaGPS}. 
Moreover, unlike commercial aviation's radar-verified surveillance, sUAS traffic management relies on cooperative position sharing via Remote Identification (Remote ID) \cite{faa_remoteid2021}, where each broadcast is itself derived from the vehicle's onboard GPS receiver. When GPS is corrupted, an sUAS's observation includes not only its own erroneous position but also the corrupted broadcasts of every vehicle in its vicinity. The resulting joint state observation may thus be adversarially perturbed in every component simultaneously, with perturbations correlated through a shared corruption source.

\begin{figure}[t]
    \centering
    \includegraphics[width=0.9\linewidth]{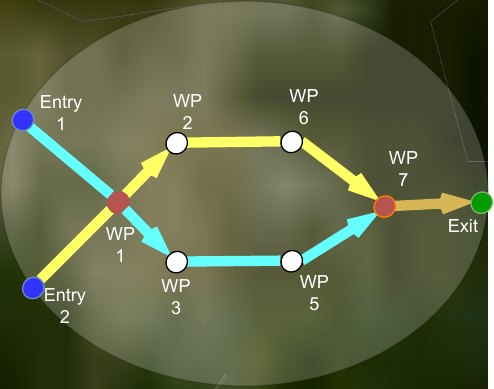}
    \caption{An en-route airspace setting for sUAS package delivery. WP = WayPoint. The UASs flying in this airspace need to maintain safe separation when they are impacted by degraded or spoofed GPS signal. The network includes two routes that first cross at \emph{WP 1} (in red), then merge later at \emph{WP 7} (light red). Each route spans about $10$ km, illustrating a compact yet realistic urban scenario for sUAS operations.}
    \label{fig:airspace}
\end{figure}

\subsection*{Related Work and Contributions}

Multi-agent reinforcement learning (MARL) has emerged as a compelling paradigm for separation assurance. Brittain et al. \cite{brittain2020deepmultiagentreinforcementlearning} demonstrated effective conflict resolution using attention-based architectures that scale gracefully to variable intruder counts, while Chen et al. \cite{chen2023integratedconflictmanagementuam} integrated tactical deconfliction with strategic traffic demand management. These approaches achieve impressive performance but assume perfect state observation, an assumption incompatible with GPS-dependent cooperative surveillance.

The vulnerability of learned policies to observation perturbations has motivated robust training methods. Adversarial approaches expose agents to gradient-based attacks \cite{GoodfellowShlensSzegedy2015,HuangEtAl2017} or learned adversaries \cite{PintoDavidsonSukthankarGupta2017, GleaveEtAl2019}. The Robust Markov Decision Process (MDP) framework captures worst-case outcomes via uncertainty sets over transitions \cite{Iyengar2005,Nilim2005}. Recent extensions address observation-level attacks: Zhang et al. \cite{zhang_robust_rl_state} introduced state-adversarial perturbations for single-agent settings, while Wang and Zou \cite{wang2022policygradientmethodrobust} developed policy gradients under $R$-contamination, where the state transition itself is corrupted with probability $R$. These state-adversarial formulations are well-suited to GPS spoofing, where vehicle dynamics are nominal, but observation is degraded. However, existing work addresses single-agent settings and does not consider the correlated full-state corruption arising in cooperative multi-agent surveillance.

Control Barrier Functions (CBFs) offer an alternative path to safety via forward-invariance guarantees \cite{ames2019}. Neural CBF extensions \cite{dawson2022} and multi-agent formulations such as GCBF++ \cite{qin2021} have shown promise. Yet, CBF methods often require known dynamics, access to true state observations, or typically centralized coordination, all of which are assumptions violated precisely when GPS is unreliable, which motivates the learning-based approach of this study.

To the best of our knowledge, no existing framework addresses robust multi-agent separation assurance under full-state observation corruption, where both the controlled aircraft's and neighboring aircraft's positions may be simultaneously degraded, while providing formal performance guarantees and supporting decentralized execution. We bridge this gap by combining $R$-contamination modeling of the complete traffic picture with analytical adversarial optimization and decision-invariance regularization. To that end, our contributions are as follows.

\begin{enumerate}
\item We cast GPS observation corruption as a zero-sum game between the adversarial and the UAS agents, and derive a closed-form approximate solution for the worst-case adversarial perturbation within the modeled uncertainty set (Theorem \ref{thm:first_order_adversary}). 
\item We prove that this closed-form adversarial perturbation approximates the worst-case perturbation defined by the formulated problem over the uncertainty set, with error that is second-order in the perturbation magnitude  (Theorem \ref{thm:remainder_bound}).  
\item We prove that Kullback-Leibler (KL)-based policy regularization between clean and adversarial observations bounds the expected performance loss induced by corrupted observations, yielding a principled robustness-performance tradeoff  (Proposition \ref{prop:performance_bound}). 
\item In high-density sUAS simulation, the robust policy achieves near-zero collision rates under corruption levels up to $35\%$, outperforming a baseline policy trained without adversarial perturbations. 
\end{enumerate}

\textit{Notation:} Bold symbols denote matrices. We write $\langle \bm{A}, \bm{B} \rangle = \sum_{ij} A_{ij}B_{ij}$ for the Frobenius inner product, $\bm{A}\odot\bm{B}$ for element-wise multiplication, and $\mathrm{sign}(\bm{A})$ for the element-wise sign operator. $\| \cdot \|_p$ is the $L_p$-norm operator. $|\cdot|_\mathrm{TV}=\frac{1}{2}\|\cdot\|_1$ denotes the total variation distance operator. Inequalities between matrices are understood \mbox{component-wise}. Absolute values are taken element-wise.  For a set $\mathcal{X}$, $\Delta(\mathcal{X})$ denotes the ($|\mathcal{X} |-1$)-dimensional probability simplex on $\mathcal{X}$. For a variable $x\in\mathbb R$ and fixed values $a,b$, $\mathrm{clip}(x, a, b)$ clamps $x$ to $[a,b]$.

\section{Problem Formulation}
\label{sec:problem}

We study autonomous separation assurance for sUAS operating in a structured low-altitude airspace, as illustrated in Fig. \ref{fig:airspace}. Each aircraft follows a predefined route through a sequence of waypoints. The control objective for an sUAS agent is to regulate its speed so as to maintain safe separation from nearby traffic while preserving efficient progress toward its destination. We formulate this problem as decentralized multi-agent reinforcement learning: $N$ homogeneous agents (with identical dynamics) apply the same learned policy using local observations, with each agent viewing itself as the \emph{ownship} and nearby aircraft as \emph{intruders}.

\subsection{States, Actions, and Dynamics} 
\label{sec:state_action_dynamics}
 
At time $t$, the true local traffic state from the ownship’s viewpoint is denoted by $\bm{S}_t \in \mathcal{S}$, which collects information relevant for separation assurance, including the ownship's progress and speed along its route, together with the relative geometric and kinematic information of nearby intruders.

For each ownship, the state is represented as $\bm S_t \in \mathbb{R}^{(1+m)\times n}$, where the first row corresponds to the ownship and the remaining rows correspond to up to $m$ nearby intruders. Since all rows of $\bm{S}_t$ are derived from GPS (both the ownship's own coordinates and the intruders' broadcasts), GPS corruption can affect the entire state matrix simultaneously, coupling observation errors across all $(1+m)$ rows.

Each aircraft state is given by $s=[x,~ y,~ \psi,~ v,~ d_g,~ u_-] \in \mathbb{R}^{1\times n}$ with $n=6$, where $(x, y)$ denotes planar position, $\psi$ heading, $v$ speed, $d_g$ distance-to-destination, and $u_-$ the previous action. The action space consists of discrete speed adjustments $\mathcal{A}=\{-\Delta v, 0, +\Delta v\}$, where $\Delta v=2.57 \ \rm{m/s}$ \cite{brittain2020deepmultiagentreinforcementlearning,faa_7110_65_speed_adjustment}. The action $u_t \in \mathcal{A}$ is a guidance-level speed command tracked by the vehicle’s inner-loop flight controller.

The aircraft state transitions are modeled by discrete-time point-mass kinematics \cite{Hoekstra2016BlueSky}
\begin{equation}
    s_{t+1} = f(s_t, u_t) = a(s_t)+b(s_t)u_t
    \label{eq:single_dynamics}
\end{equation} with 
\begin{multline}
a(s)^{\!\top}\!=\![\,x{+}(v\cos\psi{+}w_E)\Delta t,\ y{+}(v\sin\psi{+}w_N)\Delta t,\\ h(\psi),\ v,\ d_g{-}v\Delta t,\ 0\,],\\
b(s)^{\!\top}\!=\![\,\cos\psi\,\Delta t,\ \sin\psi\,\Delta t,\ 0,\ 1,\ -\Delta t,\ 1\,].
\label{eq:ab_single}
\end{multline}
The terms $(w_E,w_N)$ are wind components, $\Delta t=1\,\text{s}$ is the integration time step, and $h(\psi)$ updates the route heading at waypoint transitions. This guidance-level model captures the traffic evolution relevant to separation assurance as used in traffic-level simulations \cite{brittain2020deepmultiagentreinforcementlearning,Hoekstra2016BlueSky}.

\subsection{Markov Decision Process with Adversarial Observation Corruption}
\label{sec:mdp_formulation}

We model the robust sUAS separation and navigation problem as a Markov Decision Process (MDP) with adversarially corrupted observations. The MDP tuple $(\mathcal{S}, \mathcal{A}, P, r, \gamma)$ is described by a state space $\mathcal{S}$, an action space $\mathcal{A}$, a transition kernel $P= \{ p^u_{\bm S} \in \Delta(\mathcal{S}), u \in \mathcal{A}, \bm S \in \mathcal{S}  \}$, where $p^u_{\bm S}$ is the distribution of the next state over $\mathcal{S}$ upon taking action $u$ in state $\bm S$; a reward function $r: \mathcal{S} \times \mathcal{A} \to \mathbb{R}$; and a discount factor $\gamma \in [0, 1)$.

Let $\tilde{\bm S}_t$ denote the observation available to the agent at time $t$. This observation may be corrupted by spoofing or sensor errors, so it may be different from the true state ${\bm S}_t$. However, since the observation and the true state have the same structure, we write the agent's policy as \( \pi:\mathcal{S}\to\Delta(\mathcal{A}) \) with action sampled according to \(u_t \sim \pi(\cdot \mid \tilde{\bm S}_t) \). 
After the agent applies action $u_t$, the environment evolves to the next true state $\bm S_{t+1}\sim p_{\bm S_t}^{u_t}$ and produces reward \(r( \bm S_{t+1},u_t )\).

Unlike classical robust MDP formulations \cite{Iyengar2005,Nilim2005} which place uncertainty on the transition kernel, we consider uncertainty in the agent's observation. This choice is motivated by GPS spoofing and related sensor attacks: the physical airspace state evolves according to nominal dynamics, but the agent may act on a corrupted observation of that state. Such state-adversarial models have been studied in robust reinforcement learning \cite{zhang_robust_rl_state,wang2022policygradientmethodrobust} and are well aligned with empirical evidence on Global Navigation Satellite Systems (GNSS) spoofing in UAV systems \cite{HarshadUAVTakeOver,AndrewUAVcontrolViaGPS,psiaki2016gnss}.

We model GPS corruption through $R$-contamination uncertainty. Specifically, conditional on the true next state $\bm S_{t+1}$, the observation $\tilde{\bm S}_{t+1}$ available to the agent is generated as
\begin{equation}
    \tilde{\bm S}_{t+1} \sim O_R(\cdot \mid \bm S_{t+1}),
    \label{eq:obs_kernel}
\end{equation}
where $R\in[0,1)$ denotes the corruption probability. 
\begin{definition}
    \label{def:rcontam}
    Given the true next state $\bm S_{t+1}$, the corrupted observation follows:
    \begin{equation}
     O_R(\cdot \mid \bm S_{t+1}) = (1-R) \cdot \ind_{\bm{S}_{t+1}} + R \cdot \ind_{\bm{\Xi}_{t+1}},
    \label{eq:r_contam}
    \end{equation}
    where $\ind_{\bm{S}}$ denotes a point mass at $\bm{S}$, and the adversarial observation $\bm{\Xi}_{t+1}$ belongs to the state-dependent uncertainty set:
    \begin{equation}
    \Omega(\bm S_{t+1}) = \left\{ \bm{\Xi} \in \mathcal{S} : |\bm{\Xi} - \bm{S}_{t+1}| \leq \bm{\kappa} \right\}.
    \label{eq:uncertainty_set}
\end{equation}
Here, $\bm \kappa \in \mathbb{R}^{|\bm S|}_{\ge 0}$ specifies element-wise corruption bounds.
\end{definition}
Thus, with probability $1-R$ the next observation is accurate, and with probability $R$ it is replaced by an adversarial observation within the bounded uncertainty set. The uncertainty bound $\bm \kappa$ reflects physically plausible GPS and sensing errors.  In particular, the position components are calibrated to empirically reported urban GNSS error magnitudes \cite{Peretic2025gnsserrors}. 
This corruption model captures bounded but potentially adversarial deviations in the agent's perceived traffic configuration. 

A key modeling choice is that the reward depends on the true next state $\bm S_{t+1}$ rather than the corrupted observation $\tilde{\bm S}_{t+1}$. This reflects the fact that collisions and separation violations are physical events: they occur in the real airspace regardless of what the agent perceives. Since training is performed in simulation prior to deployment, computing such true rewards is realistic.

\subsection{Reward}
\label{sec:reward}

The reward is computed from the true next state and encourages safe separation with minimal speed changes. Specifically, we design it as

\begin{equation}
r(\bm S_{t+1},u_t)=\sum_{j=1}^{m} q\!\left(d_{oj}(\bm S_{t+1})\right)+g(u_t),
\label{eq:reward}
\end{equation}
where
\begin{equation}
d_{oj}(\bm S_{t+1})=\left\|(x^o_{t+1},y^o_{t+1})-(x^j_{t+1},y^j_{t+1})\right\|_2
\end{equation}
is the distance between the ownship $o$ and intruder $j$. The function $q(\cdot)$ penalizes unsafe proximity, i.e. loss-of-separation, while the function $g(\cdot)$ penalizes control effort.

\subsection{Robust Objective}
\label{sec:robust_objective}

Given this reward, the value function of a stationary policy $\pi$ (evaluated on the true state) is:
\begin{equation}
    V^\pi(\bm S_t) = \mathbb{E}_\pi\!\left[  \sum_{k=0}^{\infty}\gamma^k r(\bm S_{t+k+1},u_{t+k})
\,\middle|\,
\bm S_t  \right].
\label{eq:nominal_value}
\end{equation}
However, under observation corruption, the policy no longer acts on the true state $\bm S_t$ but on the corrupted observation $\tilde{\bm S}_t$. The controlled agent's performance is therefore affected through the selected action and, recursively, through the future value of the resulting trajectory.

Under Definition \ref{def:rcontam}, for a given action $u_t$, we define the one-step robust value objective as
\begin{align}
    \mathcal{J}^\pi_{\mathrm{rob}}(\bm S_t,u_t,\bm \Xi_{t+1}) &= r(\bm S_{t+1},u_t) + \gamma\Bigl[ (1-R)V^\pi(\bm S_{t+1}) \notag \\ &+ R\,V^\pi(\bm \Xi_{t+1}) \Bigr],
\label{eq:robust_successor_objective}
\end{align}
where $\bm S_{t+1}\sim p_{\bm S_t}^{u_t}$ is the true next state and $\bm \Xi_{t+1}\in\Omega(\bm S_{t+1})$ is chosen within the uncertainty set \eqref{eq:uncertainty_set}. 
The worst-case adversarial observation is then defined by
\begin{equation}
\bm \Xi_{t+1}^\star \in \underset{\bm \Xi \in \Omega(\bm S_{t+1})}{\arg\min} \mathcal{J}^\pi_{\mathrm{rob}}(\bm S_t,u_t,\bm \Xi) .
\label{eq:worst_case_xi}
\end{equation}
Since $r(\bm S_{t+1},u_t)$ and $V^\pi(\bm S_{t+1})$ do not depend on $\bm \Xi$, the minimization is equivalently
\begin{equation}
    \bm \Xi_{t+1}^\star \in \underset{\bm \Xi \in \Omega(\bm S_{t+1})}{\arg\min} \, V^\pi(\bm \Xi).
    \label{eq:worst_case_xi_equiv}
\end{equation}
This yields the one-step robust value target of policy $\pi$ 
\begin{align}
    V^\pi_{\mathrm{rob}}(\bm S_t) &= \mathbb{E}_\pi \Bigl[r(\bm S_{t+1},u_t) + \gamma\bigl((1-R)V^\pi(\bm S_{t+1}) \notag \\ &+ R \min_{\bm \Xi \in \Omega(\bm S_{t+1})} V^\pi(\bm \Xi)\bigr) \,\Big| \,\bm S_t \Bigr].
    \label{eq:robust_value}
\end{align}
The corresponding robust control objective for the aircraft is to find a policy that maximizes expected discounted return under this worst-case bounded corruption model: 
\begin{equation}
\pi^\star \in \arg\max_{\pi}~ \mathbb{E}\!\left[V^\pi_{\mathrm{rob}}(\bm S_0)\right].
\label{eq:robust_policy_objective}
\end{equation}
However, direct optimization of \eqref{eq:worst_case_xi}--\eqref{eq:robust_policy_objective} is generally intractable for nonlinear value approximators. We therefore develop a tractable policy-gradient method based on a first-order characterization of the worst-case adversarial observation, coupled with regularization that promotes perturbation-invariant policy behavior.

\section{Robust Policy Optimization Method}
\label{sec:solution}

\subsection{First-Order Worst-Case Adversarial Observation}
\label{sec:first_order_adversary}
The robust objective in Section \ref{sec:robust_objective} requires solving the inner minimization in \eqref{eq:robust_policy_objective} described by \eqref{eq:robust_value} and \eqref{eq:worst_case_xi_equiv}. 
To that end, we make the following assumption.
\begin{assumption}
\label{assump:local_diff}
    The value function $V^\pi:\mathcal{S} \to\mathbb R$ is differentiable. 
\end{assumption}
This assumption is standard for local first-order analysis and is reasonable for smooth neural value approximators on bounded neighborhoods of the state space. 
Under such assumption, a first-order Taylor expansion of $V^\pi(\bm \Xi)$ around $\bm S_{t+1}$ yields
\begin{multline}
    V^\pi(\bm \Xi) = V^\pi(\bm S_{t+1}) + \left\langle \nabla V^\pi(\bm S_{t+1}),\, \bm \Xi-\bm S_{t+1} \right\rangle\\ + o(\| \bm \Xi-\bm S_{t+1}  \|).
    \label{eq:first_order_value}
\end{multline}
Neglecting higher-order terms inside $o(\cdot)$, we approximate
\begin{multline}
    \min_{\bm \Xi \in \Omega(\bm S_{t+1})} V^\pi(\bm \Xi) \;\approx\; \min_{\bm \Xi \in \Omega(\bm S_{t+1})} \bigl[V^\pi(\bm S_{t+1}) \\ 
    + \bigl\langle\nabla V^\pi(\bm S_{t+1}), \, \bm \Xi-\bm S_{t+1} \bigr\rangle \bigr].
\label{eq:linearized_adversary_problem}
\end{multline}
This leads to a closed-form expression for the worst-case adversarial perturbation, given in the following theorem.
\begin{theorem}
    \label{thm:first_order_adversary}
    Let Assumption \ref{assump:local_diff} hold, and let the uncertainty set be componentwise bounded as in \eqref{eq:uncertainty_set}.
    Then, a minimizer of the first-order (FO) adversarial problem \eqref{eq:linearized_adversary_problem} is
    \begin{equation}
        \bm \Xi^{\star,\mathrm{FO}} = \bm S_{t+1} - \bm \kappa \odot \mathrm{sign}\bigl( \nabla V^\pi(\bm S_{t+1}) \bigr).
        \label{eq:first_order_adversarial_state}
    \end{equation}
\end{theorem}

\begin{proof}
Letting $\bm \delta=\bm \Xi-\bm S_{t+1}$, the constraint $\Xi \in \Omega(\bm S_{t+1})$ becomes $|\bm \delta|\leq \bm \kappa$. Therefore, since $V^\pi(\bm S_{t+1})$ is constant with respect to $\bm \Xi$, \eqref{eq:linearized_adversary_problem} reduces to 
\begin{equation}
    \min_{|\bm \delta|\le \bm \kappa} \left\langle \nabla V^\pi(\bm S_{t+1}),\, \bm \delta \right\rangle.
\label{eq:delta_problem}
\end{equation}
Also,
    \(
    \left\langle \nabla V^\pi(\bm S_{t+1}), \bm \delta \right\rangle = \sum_{i,j} \frac{\partial V^\pi}{\partial S_{ij}}(\bm S_{t+1})\,\delta_{ij}, \ \ |\delta_{ij}|\le \kappa_{ij}.
    \)
    Hence, each term in \eqref{eq:delta_problem} is minimized independently by taking 
    \(
    \delta_{ij}^\star = -\kappa_{ij}\,\mathrm{sign}\!\left(\frac{\partial V^\pi}{\partial S_{ij}}(\bm S_{t+1})\right). 
    \)
    Collecting the coordinates yields 
    \(
    \bm \delta^\star=-\bm \kappa \odot \mathrm{sign}\!\left(\nabla V^\pi(\bm S_{t+1})\right).
    \)
    Substituting $\bm \Xi=\bm S_{t+1} + \bm \delta$ proves \eqref{eq:first_order_adversarial_state}.
\end{proof}

Using Theorem \ref{thm:first_order_adversary}, we also immediately get a closed-form expression for the worst-case value function as follows. 
\begin{corollary}
    \label{cor:first_order_value_drop}
    Under the conditions of Theorem \ref{thm:first_order_adversary},
    \begin{equation}
        \min_{\bm \Xi \in \Omega(\bm S_{t+1})} V^\pi(\bm \Xi)~\approx~V^\pi(\bm S_{t+1}) - \bigl\| \bm \kappa \odot \nabla V^\pi(\bm S_{t+1}) \bigr\|_1. 
        \label{eq:first_order_value_drop}
    \end{equation}
\end{corollary}

\begin{proof}
    Substituting $\bm \delta^\star = -\bm \kappa \odot \mathrm{sign}(\nabla V^\pi(\bm S_{t+1}))$ into the linearized objective gives
    \begin{multline*}
    V^\pi(\bm S_{t+1})+ \left\langle \nabla V^\pi(\bm S_{t+1}),\bm \delta^\star\right\rangle\\=V^\pi(\bm S_{t+1})-\sum_{i,j} \kappa_{ij} \left| \frac{\partial V^\pi}{\partial S_{ij}}(\bm S_{t+1}) \right|,
    \end{multline*}
    which is exactly \eqref{eq:first_order_value_drop}.
\end{proof}

Next, we also quantify the accuracy of this first-order approximation, and we further assume local smoothness of the value function to that end. 

\begin{assumption}
    \label{assump:lipschitz_grad}
    The value function $V^\pi:\mathcal S \to \mathbb R$ has $L_V$-Lipschitz gradient in the neighborhood of $\bm S_{t+1}$, i.e.,
    \begin{equation}
        \|\nabla V^\pi(\bm x)-\nabla V^\pi(\bm y)\|_2 \le L_V\|\bm x-\bm y\|_2,
        \label{eq:lipschitz_grad}
    \end{equation}
for all $\bm x,\bm y$ in that neighborhood.
\end{assumption}

\begin{theorem}
    \label{thm:remainder_bound}
    Under Assumptions \ref{assump:local_diff} and \ref{assump:lipschitz_grad}, the first-order approximation error of the worst-case value over $\Omega(\bm S_{t+1})$ satisfies
    \begin{multline}
        \Bigl| \min_{\bm \Xi \in \Omega(\bm S_{t+1})} V^\pi(\bm \Xi)-\bigl(V^\pi(\bm S_{t+1})- \| \bm \kappa \odot \nabla V^\pi(\bm S_{t+1}) \|_1 \bigr) \Bigr|\\ \le \frac{L_V}{2}\| \bm \kappa \|_2^2.
        \label{eq:remainder_bound}
    \end{multline}
\end{theorem}
\begin{proof}
    Under Assumption \ref{assump:lipschitz_grad}, for any $\bm \Xi$ in a neighborhood of $\bm S_{t+1}$,
    \begin{multline}
        V^\pi(\bm \Xi) \le V^\pi(\bm S_{t+1})+\langle \nabla V^\pi(\bm S_{t+1}), \bm \Xi-\bm S_{t+1} \rangle\\+\frac{L_V}{2}\|\bm \Xi-\bm S_{t+1}\|_2^2,
        \label{eq:descent_upper}
    \end{multline}
    and similarly, 
    \begin{multline}
        V^\pi(\bm \Xi) \ge V^\pi(\bm S_{t+1}) + \langle \nabla V^\pi(\bm S_{t+1}), \bm \Xi-\bm S_{t+1} \rangle\\-\frac{L_V}{2}\|\bm \Xi-\bm S_{t+1}\|_2^2.
        \label{eq:descent_lower}   
    \end{multline}
    For any $\bm \Xi \in \Omega(\bm S_{t+1})$, we have $|\bm \Xi-\bm S_{t+1}|\le \bm \kappa$, hence \begin{equation}
        \|\bm \Xi-\bm S_{t+1}\|_2^2 \le \|\bm \kappa\|_2^2.
        \label{eq:kappa_norm_bound}
    \end{equation}
    Minimizing the linear term over $\Omega(\bm S_{t+1})$ gives the following, based on Corollary \ref{cor:first_order_value_drop}
    \begin{align}
        \min_{\bm \Xi \in \Omega(\bm S_{t+1})}\left[V^\pi(\bm S_{t+1})+ \langle \nabla V^\pi(\bm S_{t+1}), \bm \Xi-\bm S_{t+1}\rangle\right] \notag \\=V^\pi(\bm S_{t+1})-\|\bm \kappa \odot \nabla V^\pi(\bm S_{t+1})\|_1.
        \label{eq:linearized_min_value}
    \end{align}
    Combining \eqref{eq:descent_upper}-\eqref{eq:descent_lower} with \eqref{eq:kappa_norm_bound}-\eqref{eq:linearized_min_value} yields
    {\small \begin{align*}
        \min_{\bm \Xi \in \Omega(\bm S_{t+1})} V^\pi(\bm \Xi)&\le V^\pi(\bm S_{t+1})- \|\bm \kappa \odot \nabla V^\pi(\bm S_{t+1})\|_1+\frac{L_V}{2}\|\bm \kappa\|_2^2, \\
        \min_{\bm \Xi \in \Omega(\bm S_{t+1})} V^\pi(\bm \Xi)&\ge V^\pi(\bm S_{t+1})-\|\bm \kappa \odot \nabla V^\pi(\bm S_{t+1})\|_1-\frac{L_V}{2}\|\bm \kappa\|_2^2.
    \end{align*}}
Combining the two inequalities proves \eqref{eq:remainder_bound}.
\end{proof}
 
Theorem \ref{thm:remainder_bound} provides a local approximation guarantee for the first-order adversarial surrogate: the gap between the exact worst-case value and its closed-form estimate is second-order in the corruption radius. Hence, for sufficiently small bounded observation perturbations, the first-order adversarial observation captures the dominant degradation in the value function of the next state.
Assumption \ref{assump:lipschitz_grad} requires care in practice: for a piecewise-affine critic, as adopted in Section \ref{sec:robust_ppo}, $\nabla V^\pi$ is piecewise constant and no finite $L_V$ satisfies \eqref{eq:lipschitz_grad} at a region boundary. Theorem \ref{thm:remainder_bound} still applies with the bound in \eqref{eq:remainder_bound} replaced by $G\|\bm\kappa\|_2$, where $G:=\sup_{\bm\Xi\in\Omega(\bm S_{t+1})} \|\nabla V^\pi(\bm\Xi)-\nabla V^\pi(\bm S_{t+1})\|_2$; in particular $G=0$ when
$\Omega(\bm S_{t+1})$ lies within a single region, where \eqref{eq:first_order_adversarial_state} is the exact minimizer.

\subsection{Robust PPO with Shared Actor-Critic}
\label{sec:robust_ppo}

Having established a closed-form expression for the worst-case adversarial perturbation, i.e., for the minimizer in \eqref{eq:robust_value},
we solve \eqref{eq:robust_policy_objective} using proximal policy optimization (PPO) \cite{schulman2017proximalpolicyoptimizationalgorithms} with a shared actor-critic network. Training is centralized in the sense that trajectories from all agents update one shared parameter set, whereas execution is decentralized: each sUAS acts on its own local observation.

We parameterize the policy with an actor $\pi_\theta$ and the value function with a critic $V_\phi$. As shown in Fig. \ref{fig:architecture}, the actor and critic share a common observation encoder and differ only in their output heads. Training proceeds in simulation and has two phases. In the first phase, we pretrain the architecture under clean observations ($R=0$) using standard PPO, yielding a set of frozen parameters $\bar\theta, \bar\phi$. In the second phase, we initialize another set of trainable parameters $\theta,\phi$ and optimize them under $R$-contaminated observations: at each transition, the true next state $\bm S_{t+1}$ is observed with probability $1-R$, and replaced by the adversarial perturbation (derived in Subsection \ref{sec:first_order_adversary})
\begin{equation}
    \bm \Xi_{t+1}^{\star,\mathrm{FO}} = \bm S_{t+1} - \bm \kappa \odot \mathrm{sign}\bigl( \nabla V_{\bar\phi}(\bm S_{t+1}) \bigr) \label{eq:nominal_teacher_adv_state}
\end{equation}
with probability $R$. The adversary is generated from the frozen teacher critic $V_{\bar{\phi}}$ throughout. This ensures that the adversary targets consistent vulnerabilities rather than co-adapting with the policy.

The actor and critic losses follow standard PPO, computed on contaminated trajectories. Specifically, let $\hat R_t^R$ denote the return target computed from the sampled trajectory under \eqref{eq:r_contam}, and let $\hat A_t^R$ denote the corresponding advantage estimate. The value head is trained by minimizing the squared error between its prediction and the realized return impacted by the $R$-contamination: \begin{equation}
    \mathcal{L}_V(\phi)
    = \mathbb E_t\Bigl[ \bigl(V_\phi(\tilde{\bm S}_t)-\hat R_t^R\bigr)^2 \Bigr].
    \label{eq:robust_value_loss}
\end{equation}
The actor is updated by the PPO clipped surrogate
{\small
\begin{equation}
    \mathcal{L}_{\mathrm{clip}}(\theta) = -\mathbb E_t \bigg[
    \min\Bigl( \rho_t(\theta)\hat A_t^R,\,
    \mathrm{clip}\bigl(\rho_t(\theta),1- \epsilon,1+\epsilon\bigr)\hat A_t^R
    \Bigr)
    \bigg],
    \label{eq:robust_ppo_actor}
\end{equation}
}
where  \(
    \rho_t(\theta)
    =
    \frac{\pi_\theta(u_t\mid \tilde{\bm S}_t)}
         {\pi_{\theta_{\mathrm{old}}}(u_t\mid \tilde{\bm S}_t)}.
    \label{eq:ppo_ratio}\)
Here $\theta_{\mathrm{old}}$ denotes the policy parameters before the update during training and $\epsilon=0.2$ is a hyperparameter \cite{schulman2017proximalpolicyoptimizationalgorithms}. An entropy bonus is added in the usual way to encourage exploration \cite{schulman2017proximalpolicyoptimizationalgorithms}.
Note that the training structure follows standard PPO, with observations probabilistically replaced by the closed-form adversarial perturbation \eqref{eq:nominal_teacher_adv_state}. The regularization terms introduced next further encourage the policy to behave consistently under clean and corrupted observations while preserving nominal safe separation performance.

\begin{figure}[!t]
    \centering
\includegraphics[width=\linewidth]{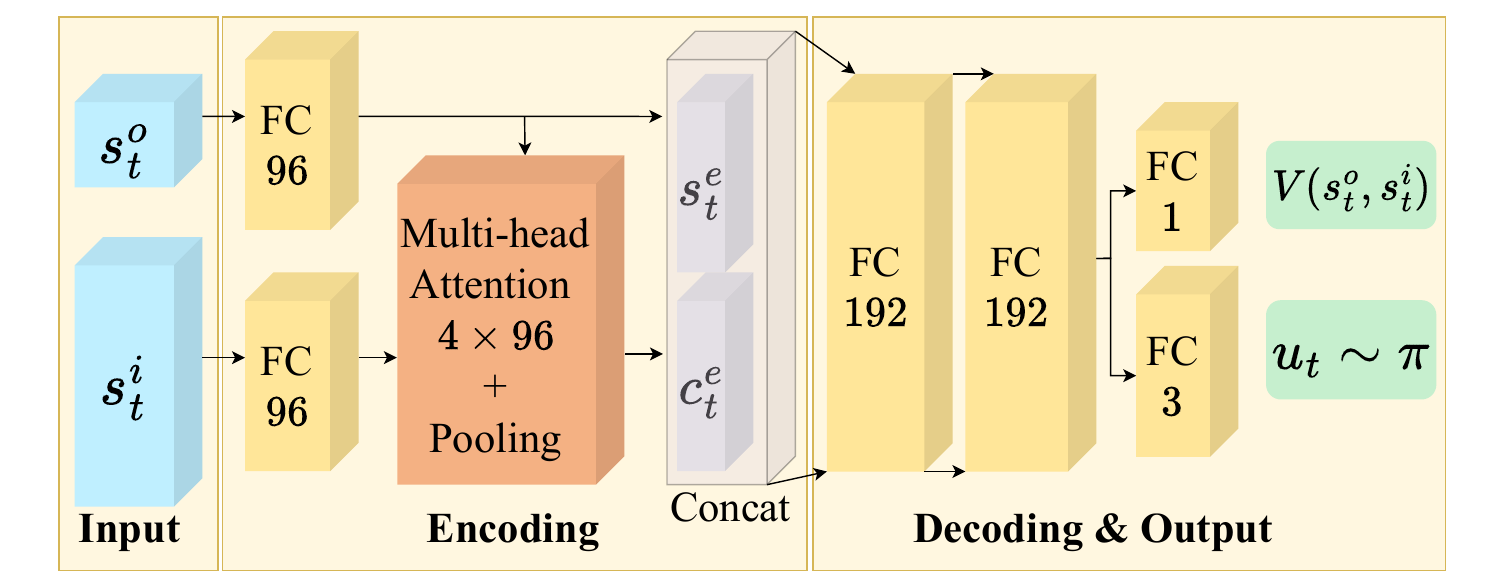}
    \caption{Neural network architecture. Ownship and intruder states are independently encoded through fully connected (FC) layers with LeakyReLU activations. Intruder embeddings are aggregated via multi-head attention with pooling. The concatenated representation branches into policy and value heads.}
    \label{fig:architecture}
\end{figure}

\subsection{Perturbation-Invariant  Policy Regularization}

We augment the robust PPO objective with two complementary Kullback-Leibler
(KL) regularizers: an invariance term that aligns the current policy on clean and adversarial observations, and an anchoring term that keeps the learned policy close to the pretrained nominal policy on clean observations.

Let $\pi_{\bar\theta}$ denote the pretrained nominal teacher policy. For the first-order adversarial observation $\bm \Xi_t^{\star,\mathrm{FO}}$ defined in \eqref{eq:nominal_teacher_adv_state}, we define the perturbation-invariance regularizer as
\begin{equation}
    \mathcal{L}_{\mathrm{inv}}(\theta) =\mathbb E_t \bigg[  D_{\mathrm{KL}}\bigl( \pi_\theta(\cdot \mid \bm S_t) ~\big\|~ \pi_\theta(\cdot\mid \bm \Xi_t^{\star,\mathrm{FO}} ) \bigr)  \bigg].
    \label{eq:inv_kl}
\end{equation}
This term penalizes the policy sensitivity to observation corruption, encouraging similar decisions under clean and perturbed inputs. Moreover, the teacher-anchor regularizer is
\begin{equation}
    \mathcal{L}_{\mathrm{anchor}}(\theta) = \mathbb E_t \bigg[   D_{\mathrm{KL}}\bigl( \pi_{\bar\theta}(\cdot\mid \bm S_t)~\big\|~ \pi_\theta(\cdot\mid \bm S_t)  \bigr)   \bigg]. 
    \label{eq:anchor_kl}
\end{equation}
This term discourages unnecessary drift from the nominal teacher on clean observations, ensuring that robustness does not come at the cost of performance. Thus, the complete actor objective becomes
\begin{equation}
    \mathcal{L}_{\mathrm{actor}}(\theta) = \mathcal{L}_{\mathrm{clip}}(\theta) +\lambda_{\mathrm{inv}}\,\mathcal{L}_{\mathrm{inv}}(\theta) + \lambda_{\mathrm{anchor}}\,\mathcal{L}_{\mathrm{anchor}}(\theta), \label{eq:full_actor_loss}
\end{equation}
where $\lambda_{\mathrm{inv}},\lambda_{\mathrm{anchor}}\ge 0$ are regularization weights.

The following proposition shows that controlling the invariance regularizer (so that $\mathbb{E}_{\bm S}\!\left[D_{\mathrm{KL}}\!\left(\pi_\theta(\cdot\mid \bm S)~\middle\|~\pi_\theta(\cdot\mid \bm \Xi^\star)\right)\right]\le B$)
directly bounds the one-step decision degradation induced by acting on corrupted observations. To that end, we  define the state-action value function for a policy $\pi$: 
\begin{equation*}
Q^\pi(\bm S_t, u_t) = \mathbb{E}[r(\bm S_{t+1}, u_t) + \gamma V^\pi(\bm S_{t+1}) \mid \bm S_t, u_t].
\end{equation*}

\begin{proposition}
\label{prop:performance_bound}
Suppose the policy satisfies
\begin{equation*}
\mathbb{E}_{\bm S}\!\left[D_{\mathrm{KL}}\!\left(\pi_\theta(\cdot\mid \bm S)~\middle\|~\pi_\theta(\cdot\mid \bm \Xi^\star)\right)\right]\le B,
\label{eq:kl_constraint}
\end{equation*}
where $\bm \Xi^\star \in \Omega(\bm S)$ is the adversarial observation and $B\ge 0$ is the expected invariance budget. Let
\begin{equation*}
V^\pi_{\mathrm{corrupt}}(\bm S)=\mathbb{E}_{u\sim \pi_\theta(\cdot\mid \bm \Xi^\star)}\bigl[Q^\pi(\bm S,u)\bigr],
\end{equation*}
denote the one-step value obtained when the action is selected from the corrupted observation, and assume
\begin{equation*}
|Q^\pi(\bm S,u)| \le Q_{\max}, \qquad \text{for all } \bm S,u.
\end{equation*}
Then
\begin{equation}
\mathbb{E}_{\bm S}\!\left[|V^\pi(\bm S)-V^\pi_{\mathrm{corrupt}}(\bm S)|\right]\le Q_{\max}\sqrt{2B}.
\label{eq:performance_bound}
\end{equation}
\end{proposition}

\begin{proof}
For any fixed $\bm S$, let $ p(\cdot)=\pi_\theta(\cdot\mid \bm S)$, and $q(\cdot)=\pi_\theta(\cdot\mid \bm \Xi^\star).$
Then,
\(
V^\pi(\bm S)-V^\pi_{\mathrm{corrupt}}(\bm S)=\sum_u \bigl(p(u)-q(u)\bigr)Q^\pi(\bm S,u).
\)

Using $|Q^\pi(\bm S,u)|\le Q_{\max}$ and the total-variation (TV) bound on expected differences:
\begin{equation}
\left|V^\pi(\bm S)-V^\pi_{\mathrm{corrupt}}(\bm S)\right|\le2Q_{\max}\,|p-q|_{\mathrm{TV}}.
\label{eq:tv_step}
\end{equation} 
By Pinsker's inequality (Lemma 12.6.1 in \cite{CoverThomas2006})
\begin{equation}
|p-q|_{\mathrm{TV}}\le\sqrt{\frac{1}{2}D_{\mathrm{KL}}(p\|q)}.
\label{eq:pinsker}
\end{equation}
Let $\Delta V(\bm S) := |V^\pi(\bm S)-V^\pi_{\mathrm{corrupt}}(\bm S)|$. Combining \eqref{eq:tv_step} and \eqref{eq:pinsker}, and then taking expectation over $\bm S$, gives
\begin{equation*}
\mathbb{E}_{\bm S}\!\left[\Delta V(\bm S)\right]\le2Q_{\max}\,\mathbb{E}_{\bm S}\!\left[\sqrt{\frac{1}{2}D_{\mathrm{KL}}\!\left(\pi_\theta(\cdot\mid \bm S)\|\pi_\theta(\cdot\mid \bm \Xi^\star)\right)}\right].
\end{equation*}
Applying Jensen's inequality to the concave square-root function yields $\mathbb{E}[\sqrt{X}] \le \sqrt{\mathbb{E}[X]},$
hence $\mathbb{E}_{\bm S}\!\left[\Delta V(\bm S)\right]\le2Q_{\max}\sqrt{\frac{B}{2}}=Q_{\max}\sqrt{2B},$
which proves \eqref{eq:performance_bound}.
\end{proof}

Proposition \ref{prop:performance_bound} bounds the one-step degradation or spoofing attack, and the next corollary generalizes the bound over the full planning horizon.

\begin{corollary}
\label{cor:robust_value}
Under the conditions of Proposition \ref{prop:performance_bound}, the expected degradation in robust next-state value under the probabilistic $R$-contamination model satisfies
\begin{equation}
\mathbb{E}\!\left[V^\pi(\bm S) -V^\pi_{\mathrm{rob}}(\bm S)\right]\le\frac{\gamma R}{1-\gamma}\,Q_{\max}\sqrt{2B}.
\label{eq:robust_value_bound}
\end{equation}
\end{corollary}

\begin{proof}
By Proposition \ref{prop:performance_bound}, each corrupted decision incurs expected loss at most $Q_{\max}\sqrt{2B}$ relative to the clean-observation action. Under $R$-contamination, corruption occurs independently with probability $R$ at each future step. Summing the discounted one-step losses yields 
\begin{equation*}
\sum_{k=1}^\infty \gamma^k R\,Q_{\max}\sqrt{2B} = \frac{\gamma R}{1-\gamma}\,Q_{\max}\sqrt{2B},
\end{equation*}
which proves \eqref{eq:robust_value_bound}.
\end{proof}
In practice, the regularization weight $\lambda_\mathrm{inv}$ in \eqref{eq:full_actor_loss} controls the achieved invariance budget $B$: larger $\lambda_{\mathrm{inv}}$ drives $\mathcal{L}_{\mathrm{inv}}$ toward zero, reducing $B$ and tightening the bounds in \eqref{eq:performance_bound} and \eqref{eq:robust_value_bound}.
Proposition \ref{prop:performance_bound} and Corollary \ref{cor:robust_value} bound the expected discounted return. These results characterize how gracefully safety performance degrades with $R$ and the invariance budget.

\subsection{Overall Training Objective and Procedure}
\label{sec:training}

The overall training objective combines the losses from the previous subsections:
\begin{align}
\mathcal{L}_{\mathrm{total}}(\theta,\phi)=
\mathcal{L}_{\mathrm{clip}}(\theta) + c_V\,\mathcal{L}_V(\phi)-
c_H\,\mathcal{H}(\pi_\theta)
 \notag \\+
\lambda_{\mathrm{inv}}\,\mathcal{L}_{\mathrm{inv}}(\theta)+\lambda_{\mathrm{anchor}}\,\mathcal{L}_{\mathrm{anchor}}(\theta),
\label{eq:total_loss}
\end{align}
where $\mathcal{L}_{\mathrm{clip}}$ is the PPO clipped surrogate in \eqref{eq:robust_ppo_actor}, $\mathcal{L}_V$ is the robust value loss in \eqref{eq:robust_value_loss}, $\mathcal{H}(\pi_\theta)$ is the entropy bonus standard in PPO to encourage exploration~\cite{schulman2017proximalpolicyoptimizationalgorithms}, and $\mathcal{L}_{\mathrm{inv}}$ and $\mathcal{L}_{\mathrm{anchor}}$ are defined in \eqref{eq:inv_kl}--\eqref{eq:anchor_kl}. The coefficients $c_V,c_H,\lambda_{\mathrm{inv}},\lambda_{\mathrm{anchor}}\ge0$ balance value fitting, exploration, perturbation invariance, and teacher anchoring.

At each iteration, trajectories from all agents are aggregated into a centralized batch, from which robust bootstrap targets, advantages, and regularization terms are computed. The shared actor-critic parameters are then updated by minimizing \eqref{eq:total_loss}. The complete training procedure is summarized in Algorithm \ref{alg:robust_ppo}.

\begin{algorithm}[t]
\caption{Robust PPO with Adversarial Observation}
\label{alg:robust_ppo}
\begin{algorithmic}[1]
\REQUIRE Multi-agent environment $\mathcal{E}$, total iterations $T$, corruption schedule $\mathcal{R}$, corruption bounds $\bm{\kappa}$
\STATE \textbf{Phase 1: Nominal pre-training}
\STATE Train nominal teacher policy $\pi_{\bar\theta}$ and value network $V_{\bar\phi}$ using standard PPO with $R=0$
\STATE Freeze $\pi_{\bar\theta}$ and $V_{\bar\phi}$
\STATE \textbf{Phase 2: Robust training}
\STATE Initialize trainable shared actor-critic parameters $(\theta,\phi)$
\FOR{$k=1$ to $T$}
    \STATE Sample corruption rate $R \sim \mathcal{R}$
    \STATE Collect pooled trajectories from all agents using the current shared policy $\pi_\theta$
    \FOR{each transition $(\tilde{\bm S}_t,u_t,\bm S_{t+1})$ in the batch}
        \STATE Compute true reward $r_t \gets r(\bm S_{t+1},u_t)$
        \STATE Compute teacher gradient $\bm g_{t+1} \gets \nabla V_{\bar\phi}(\bm S_{t+1})$
        \STATE Construct first-order adversarial observation \eqref{eq:nominal_teacher_adv_state} 
        \STATE Sample next observation under probabilistic $R$-contamination \eqref{eq:r_contam}
        \STATE Store transition $(\tilde{\bm S}_t,u_t,r_t,\tilde{\bm S}_{t+1})$
        \STATE Compute invariance penalty $\ell_{\mathrm{inv},t}$ and $\ell_{\mathrm{anchor},t}$ per \eqref{eq:inv_kl}--\eqref{eq:anchor_kl}
    \ENDFOR
    \STATE Compute contaminated-rollout return targets $\hat R_t^R$ and advantages $\hat A_t^R$ using standard PPO return / GAE estimation on the realized sampled trajectories
    \STATE Form batch losses $
    \mathcal L_{\mathrm{clip}},
    \mathcal L_V,
    \mathcal L_{\mathrm{inv}},
    \mathcal L_{\mathrm{anchor}},
    \mathcal H(\pi_\theta)$
    and $\mathcal L_{\mathrm{clip}}$ uses $\hat A_t^R$
    \STATE Update shared parameters $(\theta,\phi)$ by minimizing \eqref{eq:total_loss}
\ENDFOR
\STATE \textbf{return} robust shared policy $\pi_\theta$
\end{algorithmic}
\end{algorithm}

\section{Experimental Evaluation}
\label{sec:experiments}

We evaluate the proposed robust MARL framework in the structured en-route airspace of Fig. \ref{fig:airspace}, designed within the BlueSky air traffic simulator \cite{Hoekstra2016BlueSky}. Traffic arrivals on each route follow a Poisson process with a minimum headway of $30$\,s. Each vehicle is modeled after the Amazon MK30 sUAS, with nominal cruise speed $20$\,m/s and admissible speed range $[7.5,36]$\,m/s.

\begin{figure}[!t]
    \centering
    \includegraphics[width=\linewidth]{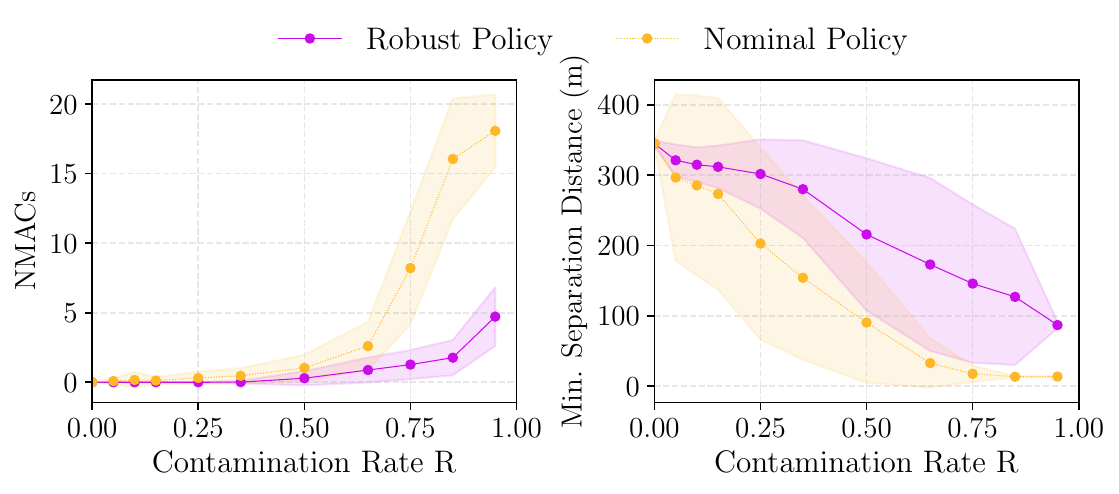}
    \caption{Safety performance under increasing observation corruption.  \emph{Left}: Near mid-air collision (NMAC) count of the small UAS. \emph{Right}: Minimum separation distance between the aircraft agents achieved per episode. The robust policy maintains near-zero NMACs through $R\approx0.35$ and degrades gracefully beyond, while the nominal policy deteriorates sharply after. Shaded regions indicate $\pm$ standard deviation over $100$ episodes.}
    \label{fig:policy_comparison}
\end{figure}

\subsection{Training Setup}
\label{sec:training_setup}

Both nominal and robust policies use the shared actor-critic architecture shown in Fig. \ref{fig:architecture}. Training is performed over $2\times 10^6$ steps. The nominal teacher policy is first trained under clean observations ($R=0$) using standard PPO. The robust policy is then trained from pooled multi-agent experience using the proposed method, with a curriculum on $
R \in \{0,~0.05,~0.15,~0.25,~0.35,~ 0.5\}.$
PPO hyperparameters follow standard settings, with additional regularization weights $\lambda_{\mathrm{inv}}=0.01$ and $\lambda_{\mathrm{anchor}}=0.01$. The learning rate is set to $2.5\times10^{-4}$ and the number of epoch $8$. 
Advantages $\hat A^R_t$ and returns $\hat R^R_t$ are computed by generalized advantage estimation (GAE) \cite{schulman2016gae} on the realized contaminated rollouts, with $\lambda=0.95$ and $\gamma=0.99$. 
The maximum number of intruders for each agent is $m=5$. The entropy and value loss coefficient are respectively $c_H=10^{-1}$ and $c_V=0.5$. The protected radius or near mid-air collision (NMAC) distance threshold is set to $d_{pz} = 100$ m, and for each sUAS, the intruder aircraft are detected within radius $d_r=500$ m. Hence the function $q(\cdot)$ and $g(\cdot)$ defined in the reward function \eqref{eq:reward} are defined as 
 \begin{align*}
     q(x) &= \begin{cases}
        0 \ \qquad\qquad \mathrm{if } \ x > d_r \\
         -\alpha + \beta x \quad \mathrm{if} \ d_{pz} < x \leqslant d_r \quad \text{(loss of separation)}\\
         -\alpha + \beta x -c_{\mathrm{NMAC}} \quad \mathrm{if} \ x \leqslant d_{pz} \quad \text{(NMAC)}
     \end{cases} \\
     g(u) &= -\lambda_u (u/\Delta v)^2,
 \end{align*}
 where $\alpha=0.1,~\beta=2\times10^{-4},~c_{\mathrm{NMAC}}=1, ~\lambda_u=0.001$ are empirically chosen and $\Delta v=2.57$ m/s (Section \ref{sec:state_action_dynamics}). 
 
 The corruption bound $\bm\kappa$ is chosen from reported urban GNSS error magnitudes (e.g., $\kappa_x=\kappa_y=60 \mathrm{m}, \kappa_v=2 \mathrm{m/s}, \kappa_\psi=5^\circ$), consistent with the uncertainty model introduced in Section \ref{sec:mdp_formulation}.

\subsection{Evaluation Protocol}
\label{sec:evaluation_protocol}

We compare the nominal policy trained without adversarial corruption against the proposed robust policy. Both policies are evaluated under the same conditions over a range of corruption rates $R\in[0,1]$. The results are averaged over $100$ simulated episodes. 

We report the following safety metrics: the number of near mid-air collisions (NMACs) and the minimum separation distance achieved during each episode on average. Lower NMAC counts and larger minimum LOS distances indicate better robustness.

\subsection{Results and Discussion}
\label{sec:main_results}

\begin{figure}
    \centering
    \includegraphics[width=\linewidth]{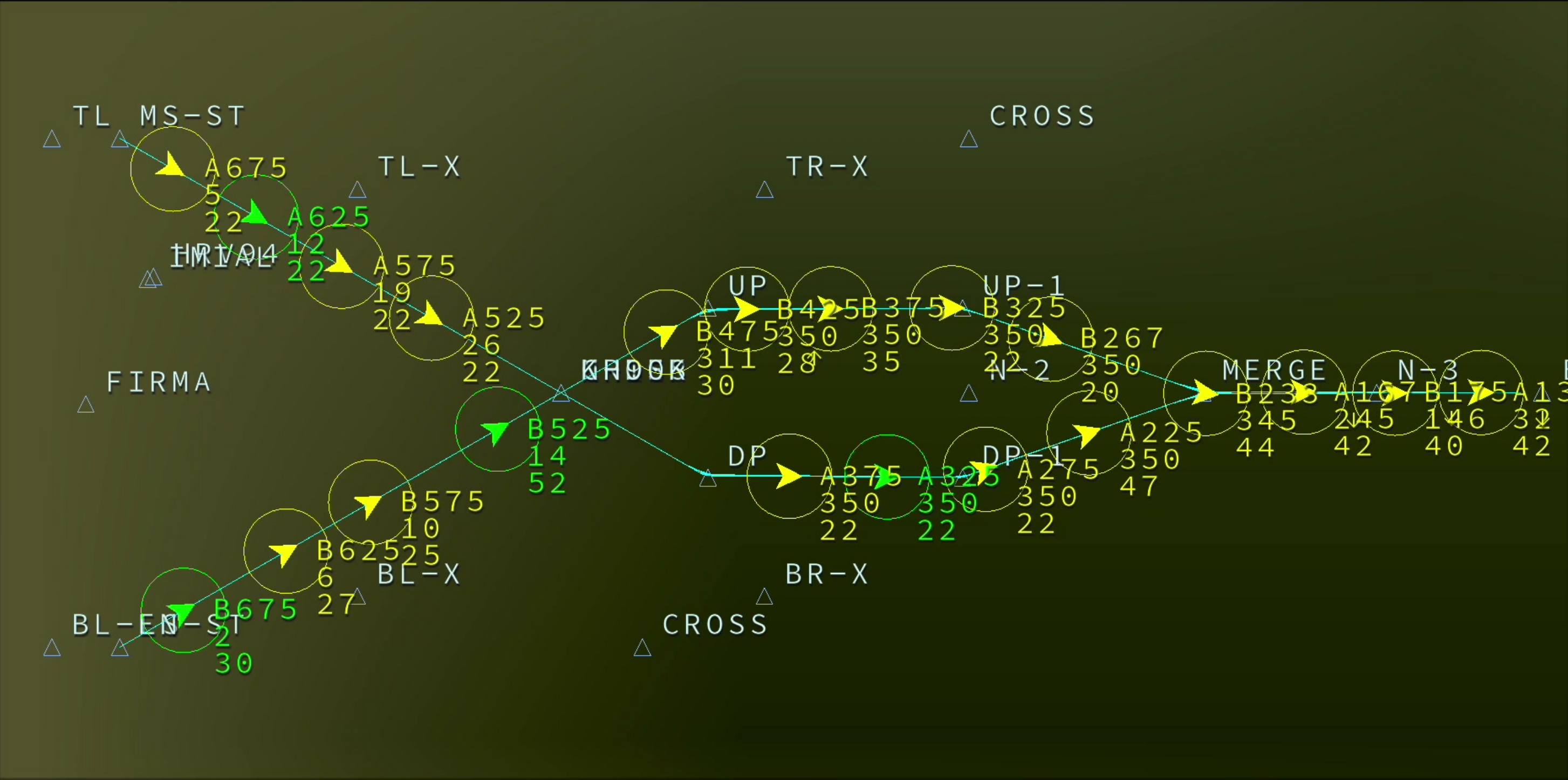}
    \caption{Snapshot of high-density sUAS traffic under the robust policy at corruption rate $R = 0.35$. Aircraft (triangles) traverse structured routes through waypoints, with separation buffers shown as circles. Yellow indicates aircraft approaching loss of separation; green indicates nominal separation status. Labels display flight ID, speed (knots), and altitude (ft). Despite adversarial GPS perturbations, all aircraft maintain safe separation.}
    \label{fig:traffic_snapshot_robust_policy}
\end{figure}

\begin{figure}
    \centering
    \includegraphics[width=\linewidth]{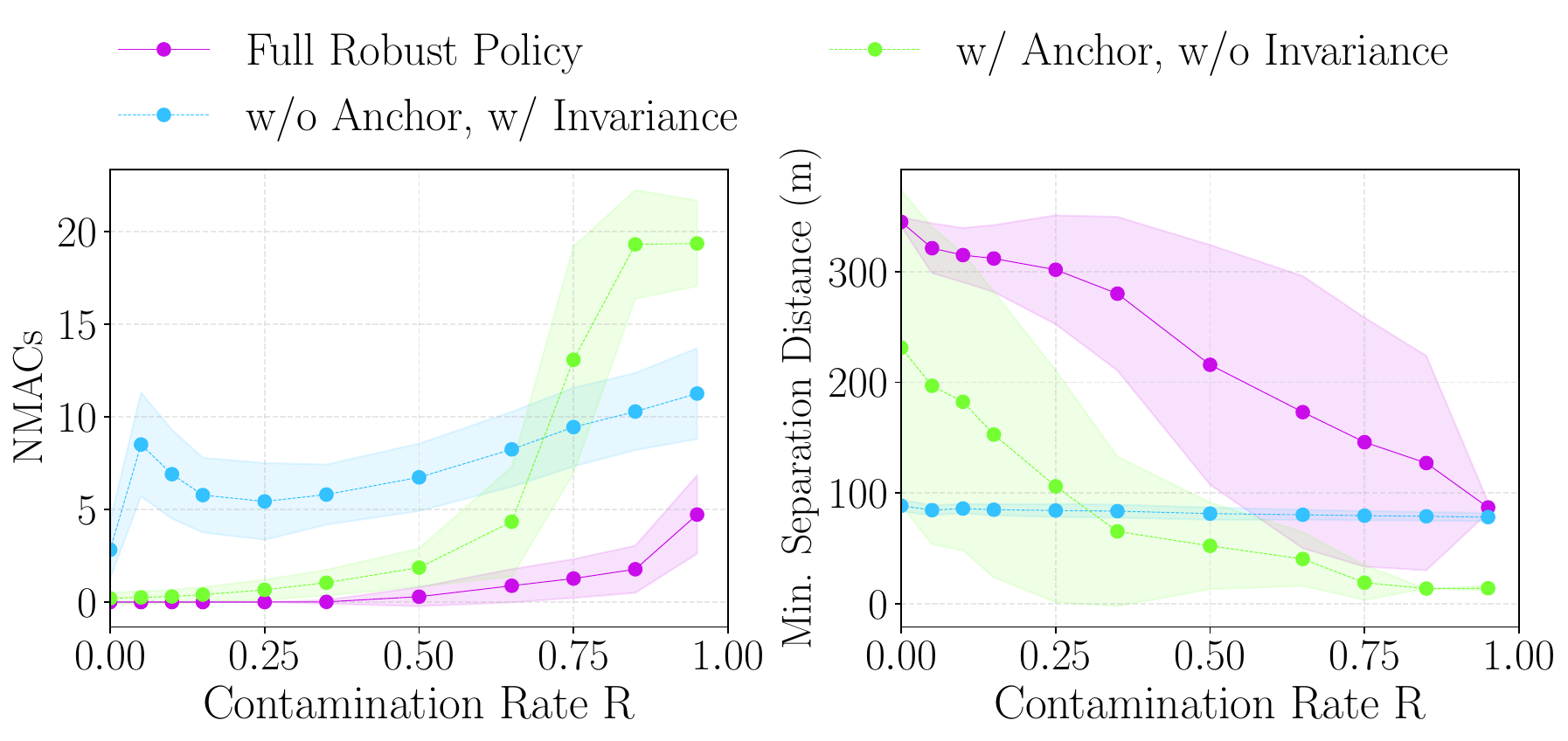}
    \caption{Ablation study isolating the contributions of invariance and anchoring regularization. \emph{Left}: NMAC count. \emph{Right}: Minimum separation distance. Invariance regularization without anchoring (blue) destabilizes training, yielding poor performance even at $R=0$. Anchoring alone (green) preserves nominal behavior but degrades sharply at high $R$. 
    The full method (purple) combines both regularizers for consistent performance. Shaded regions indicate $\pm$ standard deviation over $100$ episodes.}
    \label{fig:ablation_study}
\end{figure}

Fig. \ref{fig:policy_comparison} compares the robust and nominal policies across adversarial observation perturbation rates $R\in[0,1]$. Under clean observation ($R=0$), both policies achieve comparable safety performance, confirming that robust training does not sacrifice nominal capability. As $R$ increases, however, their behaviors diverge sharply. 
The nominal policy maintains near-zero NMACs only through $R\approx 0.25$, after which collision counts rise steeply, reaching approximately $18$ NMACs per episode at $R=0.95$. In contrast, the robust policy sustains near-zero NMACs through $R\approx0.35$ and degrades more gracefully beyond, with roughly $5$ NMACs at $R=0.95$. The minimum separation distance (right panel) tells a consistent story: the robust policy preserves larger safety margins across all corruption levels, with the gap widening as $R$ increases. Fig. \ref{fig:traffic_snapshot_robust_policy} illustrates a representative traffic scenario under the robust policy at $R=0.35$. Despite adversarial perturbations, all aircraft maintain safe separation as they traverse the structured airspace. 

\textbf{Ablation Study} Fig. \ref{fig:ablation_study} isolates the contributions of each regularization term. Two findings emerge. First, the invariance regularization without anchoring (blue) destabilizes training: NMACs are high even at $R=0$, indicating that enforcing output consistency without a stable reference collapses the learned representation. Second, anchoring alone (green) preserves nominal competence at low corruption but degrades sharply beyond $R\approx0.5$.  
The full method (purple) achieves the strongest performance while satisfying the invariance condition underlying the degradation bound of Proposition \ref{prop:performance_bound}.

\textbf{Limitations.} Three modeling choices bound the scope of these results. 
First, contamination is independent across time steps, whereas operational GNSS attacks are typically correlated. In a drag-off attack, the spoofer matches the authentic signal, captures the receiver's tracking loops, and then walks the reported position away from truth slowly enough that the implied motion remains physically plausible \cite{AndrewUAVcontrolViaGPS,psiaki2016gnss}. The resulting deviation is smooth and self-consistent rather than independently resampled. An adversary that commits to a direction over many steps is not covered by Definition \ref{def:rcontam}, and we expect the improvement over the nominal policy reported in
Fig. \ref{fig:policy_comparison} to diminish as the attack becomes more persistent.
Second, the baseline is a nominal policy rather than an alternative robust method, so the experiments establish that adversarially-informed training helps, not that the first-order adversary is preferable to iterative adversarial training or to observation-noise augmentation.
Third, the observation model assumes synchronous receipt of corrupted broadcasts, leaving communication delay, packet loss, and asynchronous Remote ID updates out of scope. Moreover, because separation is a joint outcome, an agent's advantage estimate cannot separate its own contribution from that of the aircraft against which it is resolving.


\section{Conclusion}
\label{sec:conclusion}

 We presented a robust MARL framework for sUAS separation assurance under GPS degradation and spoofing. The key contributions are a closed-form adversarial perturbation that replaces the iterative inner optimization of adversarial training with a single gradient evaluation while approximating the worst-case perturbation to second order, and a KL-based regularization that provably bounds performance degradation under observation corruption. These are integrated into a PPO algorithm with centralized training but decentralized execution.
 The simulation experiments demonstrate the effectiveness of the robust MARL policy and its superiority compared to a baseline policy in maintaining safe separation between the sUASs under GPS degradation and spoofing.

\bibliography{references}

@book{CoverThomas2006,
  author    = {Cover, Thomas M. and Thomas, Joy A.},
  title     = {Elements of Information Theory},
  publisher = {Wiley-Interscience},
  year      = {2006},
  edition   = {2nd},
  address   = {Hoboken, New Jersey},
  isbn      = {978-0-471-24195-9},
}

@inproceedings{zhang_robust_rl_state,
author = {Zhang, Huan and Chen, Hongge and Xiao, Chaowei and Li, Bo and Liu, Mingyan and Boning, Duane and Hsieh, Cho-Jui},
title = {{Robust deep reinforcement learning against adversarial perturbations on state observations}},
year = {2020},
isbn = {9781713829546},
booktitle = {Proceedings of the 34th International Conference on Neural Information Processing Systems}
}

@article{psiaki2016gnss,
  author  = {Psiaki, Mark L. and Humphreys, Todd E.},
  title   = {{GNSS} Spoofing and Detection},
  journal = {Proceedings of the IEEE},
  volume  = {104},
  number  = {6},
  pages   = {1258--1270},
  year    = {2016},
  doi     = {10.1109/JPROC.2016.2526658}
}

@online{faa_part135_2023,
  author       = {{Federal Aviation Administration}},
  title        = {{Package Delivery by Drone (Part 135)}},
  year         = {2023},
  urldate      = {2025-08-30},
  organization = {U.S. Department of Transportation}
}

@online{faa_remoteid2021,
    author = {{Federal Aviation Administration}},
    title = {{Remote Identification of Unmanned Aircraft (14 CFR Part 89)}},
    year = {2021},
    urldate      = {2025-03-19},
    organization = {U.S. Department of Transportation}

}

@online{marquand_bvlos_dallas_2024,
  author  = {Marquand, Y. L.},
  title   = {{FAA Authorises Zipline and Wing for BVLOS Operations in Dallas}},
  year    = {2024},
  urldate = {2025-08-30},
  organization = {Revolution.Aero}
}

@inproceedings {HarshadUAVTakeOver,
    author = {Harshad Sathaye and Martin Strohmeier and Vincent Lenders and Aanjhan Ranganathan},
    title = {{An Experimental Study of {GPS} Spoofing and Takeover Attacks on {UAVs}}},
    booktitle = {31st USENIX Security Symposium (USENIX Security 22)},
    year = {2022},
    pages = {3503--3520},
    month = aug,
    publisher = {USENIX Association}
}

@INPROCEEDINGS{ames2019,
  author={Ames, Aaron D. and Coogan, Samuel and Egerstedt, Magnus and Notomista, Gennaro and Sreenath, Koushil and Tabuada, Paulo},
  booktitle={2019 18th European Control Conference (ECC)}, 
  title={{Control Barrier Functions: Theory and Applications}}, 
  month = jun,
  year={2019},
    volume={},
  number={},
  pages={3420-3431},
  doi={10.23919/ECC.2019.8796030}
}

@ARTICLE{dawson2022,
  author={Dawson, Charles and Gao, Sicun and Fan, Chuchu},
  journal={IEEE Transactions on Robotics}, 
  title={{Safe Control With Learned Certificates: A Survey of Neural Lyapunov, Barrier, and Contraction Methods for Robotics and Control}}, 
  year={2023},
  volume={39},
  number={3},
  pages={1749-1767},
  doi={10.1109/TRO.2022.3232542}}

@inproceedings{qin2021,
title={{Learning Safe Multi-agent Control with Decentralized Neural Barrier Certificates}},
author={Zengyi Qin and Kaiqing Zhang and Yuxiao Chen and Jingkai Chen and Chuchu Fan},
booktitle={International Conference on Learning Representations},
month = jan,
year={2021},
}

@article{AndrewUAVcontrolViaGPS,
author = {Kerns, Andrew J. and Shepard, Daniel P. and Bhatti, Jahshan A. and Humphreys, Todd E.},
title = {{Unmanned Aircraft Capture and Control Via GPS Spoofing}},
month = jul,
year = {2014},
publisher = {John Wiley and Sons Ltd.},
journal = {J. Field Robot.},
volume = {31},
number = {4},
pages = {617--636},

}

@article{brittain2020deepmultiagentreinforcementlearning,
  author    = {Mark Brittain and Xuxi Yang and Peng Wei},
  title     = {{Autonomous Separation Assurance with Deep Multi-Agent Reinforcement Learning}},
  journal   = {Journal of Aerospace Information Systems},
  volume = {18},
  number = {12},
  pages = {890--905},
  year      = {2021},
  publisher = {AIAA},
}

@online{faa_7110_65_speed_adjustment,
  title        = {{FAA Order JO 7110.65BB: Air Traffic Control ---Section 7. Speed Adjustment}},
  author = {{Federal Aviation Administration}},
  year = {2025},
  organization = {U.S. Department of Transportation},
  urldate      = {2026-03-04}
}

@article{chen2023integratedconflictmanagementuam,
  author={Chen, Shulu and Evans, Antony D. and Brittain, Marc and Wei, Peng},
  journal={IEEE Transactions on Intelligent Transportation Systems}, 
  title={{Integrated Conflict Management for UAM With Strategic Demand Capacity Balancing and Learning-Based Tactical Deconfliction}}, 
  month = Aug,
  year={2024},
  volume = {25},
  number = {8},
  pages = {10049--10061},
}

@inproceedings{Hoekstra2016BlueSky,
  author    = {Jacco M. Hoekstra and Joost Ellerbroek},
  title     = {{BlueSky ATC Simulator Project: An Open Data and Open Source Approach}},
  booktitle = {Proceedings of the 7th International Conference on Research in Air Transportation (ICRAT)},
  year      = {2016},
  pages     = {1--8},
  publisher = {FAA/Eurocontrol USA/Europe},
}

@article{schulman2017proximalpolicyoptimizationalgorithms,
      title={{Proximal Policy Optimization Algorithms}}, 
      author={John Schulman and Filip Wolski and Prafulla Dhariwal and Alec Radford and Oleg Klimov},
      journal={eprint arXiv:1707.06347},
      year={2017},      
}

@inproceedings{Gutierrez2024multipath,
  author       = {Gutierrez, Julian and Gilabert, Russell and Dill, Evan and Hernandez, Guillermo and Kaeli, David and Closas, Pau},
  title        = {{Multipath Mitigation via Clustering for Position Estimation Refinement in Urban Environments}},
  booktitle    = {{ION Pacific PNT Conference}},
  month = {April},
  year         = {2024},
  institution  = {NASA Langley Research Center},
  volume = {},
  pages = {556-568},
}

@inproceedings{Peretic2025gnsserrors,
  author       = {Peretic, Matt and Gilabert, Russell and Carroll, Johnson and Gutierrez, Julian and Moore, Andrew and Christie, Jonathan and Dill, Evan T.},
  title        = {{Statistical Analysis of GNSS Multipath Errors in Urban Canyons}},
  booktitle    = {{IEEE/ION Position, Location and Navigation Symposium (PLANS)}},
  year         = {2025},
  volume={},
  number={},
  pages={1216-1225},
  doi={10.1109/PLANS61210.2025.11028411},
  institution  = {NASA Langley Research Center and The MITRE Corporation},
}

@InProceedings{wang2022policygradientmethodrobust,
  title = 	 {{Policy Gradient Method For Robust Reinforcement Learning}},
  author =       {Wang, Yue and Zou, Shaofeng},
  booktitle = 	 {Proceedings of the 39th International Conference on Machine Learning},
  year = 	 {2022},
  month = jun,
  pages = {23484--23526},
}

@inproceedings{GoodfellowShlensSzegedy2015,
  title     = {{Explaining and Harnessing Adversarial Examples}},
  author    = {Goodfellow, Ian J. and Shlens, Jonathon and Szegedy, Christian},
  booktitle = {Proceedings of the 3rd International Conference on Learning Representations (ICLR)},
  year      = {2015},
}

@inproceedings{schulman2016gae,
  title={{High-Dimensional Continuous Control Using Generalized Advantage Estimation}},
  author={Schulman, John and Moritz, Philipp and Levine, Sergey and
          Jordan, Michael I. and Abbeel, Pieter},
  booktitle={International Conference on Learning Representations (ICLR)},
  year={2016}
}

@inproceedings{HuangEtAl2017,
  title     = {{Adversarial Attacks on Neural Network Policies}},
  author    = {Sandy Huang and Nicolas Papernot and Ian Goodfellow and Yan Duan and Pieter Abbeel},
  booktitle = {Proceedings of the International Conference on Learning Representations Workshops},
  month = feb,
  year      = {2017},
}

@article{Iyengar2005,
  title   = {{Robust Dynamic Programming}},
  author  = {Iyengar, Garud N.},
  journal = {Mathematics of Operations Research},
  volume  = {30},
  number  = {2},
  pages   = {257--280},
  year    = {2005},
}

@article{Nilim2005,
  title   = {{Robust Control of Markov Decision Processes with Uncertain Transition Matrices}},
  author  = {Nilim, A. and El Ghaoui, L.},
  journal = {Operations Research},
  volume  = {53},
  number  = {5},
  pages   = {780--798},
  year    = {2005},
}

@inproceedings{GleaveEtAl2019,
  title={{Adversarial policies: Attacking deep reinforcement learning}},
  author={Gleave, Adam and Dennis, Michael and Wild, Cody and Kant, Neel and Levine, Sergey and Russell, Stuart},
  booktitle={International Conference on Learning Representations},
  year={2020},
}

@inproceedings{PintoDavidsonSukthankarGupta2017,
  title     = {{Robust Adversarial Reinforcement Learning}},
  author    = {Pinto, Lerrel and Davidson, James and Sukthankar, Rahul and Gupta, Abhinav},
  booktitle = {Proceedings of the 34th International Conference on Machine Learning (ICML)},
  series    = {Proceedings of Machine Learning Research},
  volume    = {70},
  pages     = {2817--2826},
  year      = {2017}
}
\bibliographystyle{IEEEtran}

\end{document}